\documentclass{article}

\PassOptionsToPackage{numbers, compress}{natbib}

\usepackage[preprint]{neurips_2020}

\usepackage[utf8]{inputenc} 
\usepackage[T1]{fontenc}    
\usepackage{hyperref}       
\usepackage{url}            
\usepackage{booktabs}       
\usepackage{amsfonts}       
\usepackage{nicefrac}       
\usepackage{microtype}      

\usepackage{amssymb}
\usepackage{amsthm}
\usepackage{graphicx}

\usepackage{algorithm}
\usepackage[noend]{algpseudocode}
\usepackage{amssymb}

\usepackage{subcaption}

\theoremstyle{plain}

\theoremstyle{plain}
\newtheorem{claim}{Claim}[section]
\theoremstyle{definition}

\theoremstyle{remark}

\title{StructureBoost: Efficient Gradient Boosting for Structured Categorical Variables}

\author{%
  Brian Lucena \\
  Numeristical\\
  \texttt{brian@numeristical.com} \\
}

\begin{document}

\maketitle

\begin{abstract}
Gradient boosting methods based on Structured Categorical Decision Trees (SCDT) have been demonstrated to outperform numerical and one-hot-encodings on problems where the categorical variable has a known underlying structure.  However, the enumeration procedure in the SCDT is infeasible except for categorical variables with low or moderate cardinality.  We propose and implement two methods to overcome the computational obstacles and efficiently perform Gradient Boosting on complex structured categorical variables.  The resulting package, called StructureBoost, is shown to outperform established packages such as CatBoost and LightGBM on problems with categorical predictors that contain sophisticated structure.  Moreover, we demonstrate that StructureBoost can make accurate predictions on unseen categorical values due to its knowledge of the underlying structure.
\end{abstract}

\section{Introduction and Background}
In supervised learning, it is often desirable to use a categorical variable as a predictor.  Traditionally, this is handled by the use of ``dummy" variables, also known as {\em one-hot encoding}, where a binary variable is created for each possible value of the categorical variable.  This approach effectively assumes that the different values of the categorical variable have no relationships of interest -- they are different, incomparable, and nothing more.  It also has the drawback of creating a large number of variables out of a single entity.

Due to the messiness of one-hot encoding, numerical encoding has emerged as an alternative.  In this scheme, each value of the categorical variable is mapped to a number, and the result is treated as a single numerical variable.  This is frequently a large improvement over one-hot encoding. Specifically, it implies an {\em ordinal} structure on the values of the categorical variable.  For example, if the make of a car is used as a categorical variable, it is useful to impose a ``ranking" that puts more expensive models on one end and cheaper models at the other end.  A simple numerical encoding technique is to use the average value of the target variable among all instances where the categorical variable attains a particular value.  This approach, or more sophisticated variants designed to reduce overfitting, lies beneath gradient boosting packages such as LightGBM~\cite{ke2017lightgbm} and CatBoost~\cite{NIPS2018_7898}~\cite{dorogush2018catboost}.

However, as illustrated in~\cite{categstruct}, a linear ranking may not effectively capture the underlying structure of the categorical variable.  The months of a year have a circular structure, while the US States have a geographical structure.  Image classification categories and diseases have a hierarchical structure.  Supervised learning methods should exploit this structure to make better predictions.

Gradient boosting~\cite{Friedman:2002:SGB:635939.635941}~\cite{Friedman00greedyfunction} has proven to be one of the most effective methods for supervised learning on tabular data and is widely used for predictive models in industry.  The development of XGBoost~\cite{chen2016xgboost} was a huge step forward, making advances in its mathematical formulation, efficiency, and smart handling of missing data.  In the last several years, packages such as CatBoost~\cite{NIPS2018_7898}~\cite{dorogush2018catboost}  and LightGBM~\cite{ke2017lightgbm} have implemented further techniques to increase efficiency and handle categorical variables in a more direct and sophisticated matter.

More recently,~\cite{categstruct} introduced the notion of a Structured Categorical Decision Tree (SCDT).which generalizes the standard decision tree framework to incorporate the structure of a categorical variable, encoded by a {\em terrain}.  The terrain of a categorical variable defines what values it would make sense to average over (or aggregate together), and thereby defines structure over the different values.  The decision tree then searches over {\em maximally coarse} splits which respect the terrain.  While the precise definition of a terrain can be tedious, it was also shown that a graph (where the vertices represent the values of the variable and edges represent ``neighboring" values) can concisely define a terrain.  In this representation, the connected sets of the graph represent these ``averageable" sets.  Note that this is different than the use of graphs in probabilistic graphical models.  The structure is defined within the values of a single variable, not to imply dependencies across variables.

The SCDT algorithm for graphical terrains requires the enumeration of all connected sets in the graph such that the complement is also a connected set -- these comprise the ``allowable" splits with respect to the terrain. However, this task is computationally intensive, even on moderately-sized graphs and so the algorithm is demonstrated only on graphs of 9 and 20 vertices.  But an interesting detail of those results demonstrates a potential path to efficiency.  The authors note that it was not necessary to evaluate every split to get good performance.  In fact, performance was improved by searching only a small subset of the allowable splits.  If the boosting procedure performs as well or better when only considering a small fraction of splits, then it may not be necessary to enumerate the full universe of allowable splits.  Rather, we need only generate a random sample of them.

In this paper, we propose two different methods for choosing a random sample of allowable splits without fully enumerating the entire space.  We implement these methods in a Gradient Boosting package called StructureBoost and demonstrate that they outperform the fully enumerated approaches in both time and predictive performance.  Furthermore, we show that StructureBoost can outperform state-of-the-art packages such as CatBoost and LightGBM on problems where the categorical predictors contain sophisticated structure that is not ordinal in nature.  Finally, we show that StructureBoost can make better predictions on unseen categorical values due to its knowledge of their structure.

The rest of the paper is organized as follows:  In Section~\ref{categstructrev} we review the notion of categorical structure introduced in~\cite{categstruct}.  In Section~\ref{splitsample} we propose two methods for sampling from the space of ``allowable" partitions without fully enumerating all of them.  Section~\ref{experiments} presents the experimental framework and results.  We then offer concluding remarks and a discussion of broader impact.

\section{Review of Categorical Structure}
\label{categstructrev}

To simplify the presentation, we restrict our attention to categorical variables with graph-based terrains.  That is, each categorical predictor has an associated graph where the vertices are the possible states of the random variable, and edges represent ``neighboring" values.  A standard ``ordinal" categorical variable would therefore be represented by a chain graph.  The months of the year would be best represented by a simple cycle, and the lower 48 US States (plus the District of Columbia) would be represented by the US-49 graph in Figure~\ref{fig:US-States}.

The central idea in the SCDT is to evaluate the {\em maximally coarse} partitions which respect the terrain of the variable.  For graph-based terrains, these ``allowable" splits are of the form $x \in S$ vs $x \in S^C$ (the complement of $S$) for all sets $S$ such that both $S$ and $S^C$ are connected sets in $G$.  For example, in the US-49 graph, the split $\{S_1, S_1^C\}$ where $S_1 = \{CT, MA, RI, NH, VT, ME\}$ would be allowable since $S_1$ and $S_1^C$ are each connected.  However, the split $\{S_2, S_2^C\}$ where $S_2 = \{NY, NJ, PA\}$ is not allowable since $S_2^C$ is not a connected set in $G$.

Simply enumerating, let alone evaluating, the allowable splits becomes computationally burdensome for graphs such as the US-49 graph.  In fact, there are 4,149,721 different allowable splits for that graph.~\cite{categstruct} Moreover, enumerating those splits required enumerating 35,327,031 connected sets of size less than or equal to 24, and then checking to see if the complements of those sets were themselves connected.  The question of how to enumerate the connected sets in a graph most efficiently is itself a subject of active research in graph theory (see e.g.~\cite{komusiewicz2019enumerating}~\cite{elbassioni2015polynomial}). While there may be clever methods to reduce this computation, merely enumerating 4 million splits is computationally burdensome.  Therefore, we propose methods to choose a random sample of allowable splits without enumerating all of them.

\begin{figure}[b]
\begin{subfigure}{.33\textwidth}
  \centering
  \includegraphics[width=.95\linewidth]{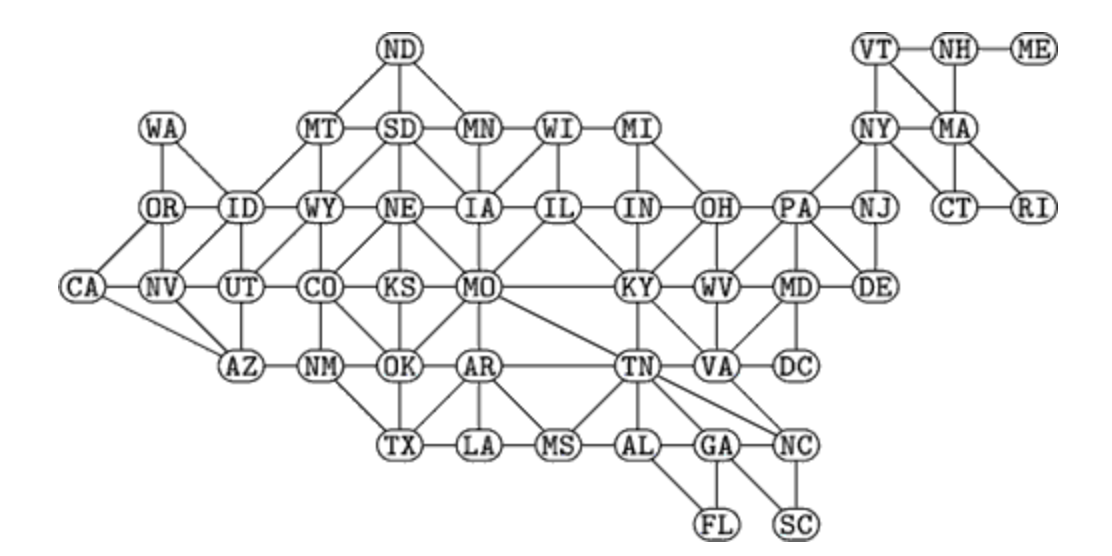}  
  \caption{US-49 graph~\cite{USgraph}}
  \label{fig:US-States}
\end{subfigure}
\begin{subfigure}{.33\textwidth}
  \centering
  \includegraphics[width=.95\linewidth]{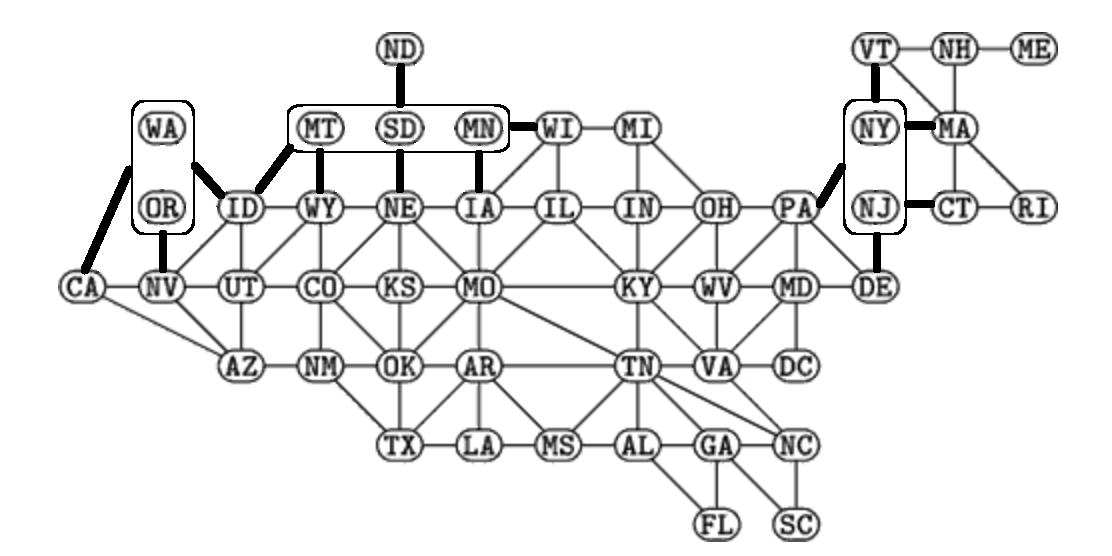}  
  \caption{After contracting 4 edges}
  \label{fig:US-States-4c}
\end{subfigure}
\begin{subfigure}{.33\textwidth}
  \centering
  \includegraphics[width=.95\linewidth]{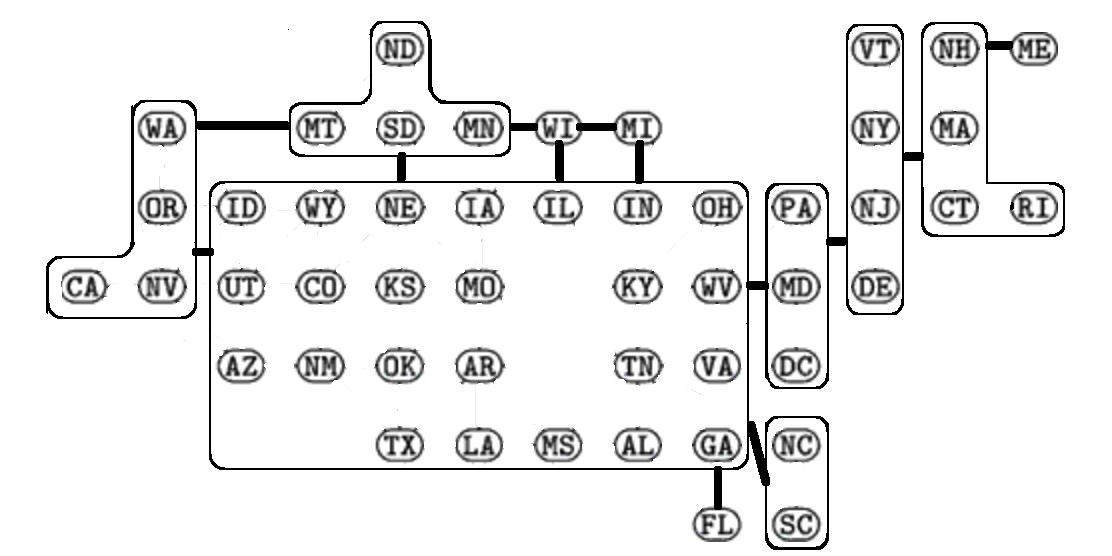}  
  \caption{Contracted to 11 vertices}
  \label{fig:US-States-contr}
\end{subfigure}
\caption{Adjacency graphs before and after edge contraction}
\label{fig:fig}
\end{figure}

\section{Efficient split sampling}
\label{splitsample}
In this section, we discuss two methods for generating random samples of allowable splits, which use techniques and concepts from graph theory.  Both methods share the following properties:

\begin{enumerate}
\item{Every allowable split has positive probability of being sampled.}
\item{Some allowable splits may be more likely than others to be sampled - we generally do not choose each allowable split with equal probability}
\item{The methods contain nuisance parameters that can ``tune" between a computationally costly extreme that approaches (or achieves) full enumeration, and very quick alternative that chooses a very small number of splits.}
\end{enumerate}

\subsection{Edge Contraction}
The first method we will present is based on edge contraction.  Edge contraction is deeply connected with the theory of graph minors developed by Robertson and Seymour (e.g.~\cite{robertson1986graph}~\cite{robertson1990graph}).  Contracting an edge means that we identify both endpoints of an edge as a single vertex and remove any duplicate edges.  This reduces the size of the graph, and can be thought of as a ``coarser" representation of the same information.  For example, in the US-49 graph of the lower 48 US states plus DC, contracting the edge between Washington and Oregon is equivalent to treating Washington and Oregon as a single state.   Every time we contract an edge in a graph, we reduce number of vertices by one, and the number of edges by at least one (and typically more).  Figure~\ref{fig:US-States-4c} shows the result after contracting the edges $\{WA,OR\},\{MT,SD\},\{SD,MN\}$, and $\{NY,NJ\}$.  By sequentially performing edge contraction 38 times starting from the US-49 graph, we would have a graph on 11 vertices, where each vertex represents a conglomeration of one or more states, such as in Figure~\ref{fig:US-States-contr}.

Given this reduced 11-vertex graph, it is now computationally inexpensive to enumerate all allowable splits.  In general, the number of allowable splits of a $k$-vertex graph is bounded above by $2^{k}$.  However, the actual number will typically be much lower for relatively sparse graphs (such as planar graphs).  For the graph shown in Figure~\ref{fig:US-States-contr}, there are only 16 allowable splits. It is easy to see how an allowable split on the contracted graph corresponds to an allowable split on the original graph, by identifying the conglomerated vertices with the set of original vertices that were merged.  

We present the details below in Algorithm~\ref{alg1}.  This method has two parameters: the size to which we contract the graph ({\em contraction\_size} or {\em c}), and the number of splits to search after we exhaustively enumerate the allowable splits of the contracted graph ({\em max\_splits\_to\_search} or {\em m}).  

\begin{algorithm}
\caption{Split Sampling via Edge Contraction}\label{alg1}
\begin{algorithmic}[0]
\State {\bf Input:} A categorical variable $X$ with an associated graph $G_0 = (V_{G_0},E_{G_0})$).
\State {\bf Parameters:} $m, c$
\State {\bf Output:} A set $\mathcal{S}$ of allowable splits such that $|\mathcal{S}| \leq m$
\State $r \Leftarrow |V_{G_0}| - c$ 
\For{$i$ in $1,2,\ldots, r$:} \Comment{Contract the graph down to size $c$}

\State Choose an edge $e$ at random from $E_{G_{i-1}}$ 
\State $G_i \Leftarrow $ the graph resulting from contracting edge $e$ in graph $G_{i-1}$ 
\EndFor

\State Enumerate all allowable splits on $G_r$, call the resulting set of splits $S^{\prime}$.
\State Let $S_0$ be the set of splits on $V_0$ corresponding to $S^{\prime}$ (by identifying  conglomerated vertices in $G_r$  \phantom{x}\hspace{70 mm} with the corresponding set of vertices in $G_0$)

\State  $S \Leftarrow$ a random subset of $S_0$ of size $\min(m, |S_0|)$.

\State Return $S$
\end{algorithmic}
\end{algorithm}

Running this algorithm for a graph $G=(V,E)$ requires $(|V| - c)$ edge contractions, each of which takes $O(|E|)$ operations to execute in a naive implementation.  This is then followed by the full enumeration of allowable splits on the contracted graph, which is generally exponential in the size of that smaller graph.  By choosing a relatively small $c$ value, we ensure that the exponential step is run on a small graph, rendering it feasible.  Note that if $c$ is the size of the original graph and $m$ is large enough, then Algorithm~\ref{alg1} reduces to full exhaustive enumeration.   

\begin{claim}
If $s = (V_a, V_b)$ is an allowable split in $G$, then with positive probability (w.p.p.), Algorithm~\ref{alg1} will output a set $S$ such that $s \in S$.
\end{claim}
\begin{proof}
Since $s$ is an allowable split, both $V_a$ and $V_b$ are connected in $G$.  Let $E^{\prime}$ be the set of edges in $G$ that have endpoints in both $V_a$ and $V_b$.  Let $G_i$ be as in Algorithm~\ref{alg1}.  If $|V_{G_i}| > 2$ then there exist edges in $E_{G_i} \setminus E^{\prime}$ which are chosen w.p.p.  Therefore, w.p.p. we never choose an edge in $E^{\prime}$ to be contracted.  So $s \in S^{\prime}$ (and therefore in $S$) w.p.p.
\end{proof}

\subsection{Spanning Trees}
 The second method utilizes the concept of a spanning tree.  To review, a tree is a connected graph which contains no cycles and a spanning tree $T$ of a graph $G$ is a subgraph of $G$ such that $V_T = V_G$ and $T$ is a tree.  Spanning trees are widely studied in graph theory and have many applications in network theory.  Particularly relevant for this application is the existence of an efficient algorithm for choosing a random spanning tree {\em uniformly} from the space of all spanning trees due to Wilson~\cite{wilson1996generating}.  
 
One property of trees is that the removal of any edge of a tree $T$ disconnects the graph into exactly two connected components $T_1$ and $T_2$.  Since these sets are connected in $T$ they are also connected in $G$.  Thus, the spanning tree $T$ specifies $m-1$ allowable splits of $G$, one for each edge in $T$.  We could choose a single edge at random, or evaluate all splits corresponding to edges in $T$.  The latter confers some computational efficiencies due to the tree structure, so that is the approach we take here.  The parameter $n$ represents the number of spanning trees we sample.  The details are given in Algorithm~\ref{alg2}.

\begin{algorithm}
\caption{Split Sampling via Spanning Trees}\label{alg2}
\begin{algorithmic}[0]
\State {\bf Input:} A categorical variable $X$ with an associated graph $G = (V_{G},E_{G})$).
\State {\bf Parameters:} $n$
\State {\bf Output:} A set $\mathcal{S}$ of allowable splits
\State $S = \emptyset$
\For{$i$ in $1,2,\ldots, n$:} 
\State Generate a spanning tree $T$ on $G$ using Wilson's algorithm
\For{each edge $e$ in $T$}
\State Let $(V_a, V_b)$ be the two connected sets of vertices that result from removing edge $e$ from $T$.
\State $S \Leftarrow S \cup \{(V_a, V_b)\}$
\EndFor
\EndFor
\State Return $S$
\end{algorithmic}
\end{algorithm}

For a graph $G=(V,E)$, Wilson's algorithm runs in time $O(|E|^{4/3})$~\cite{DBLP:journals/corr/abs-1711-06455}.  We run this step $n$ times, and for each step, visit each of the $|V|-1$ edges in the spanning tree to generate a split.  Therefore the complexity of this sampling scheme is $O(n\cdot|V||E|^{4/3})$

\begin{claim}
If $s = (V_a, V_b)$ is an allowable split in $G$, then with positive probability (w.p.p.), Algorithm~\ref{alg2} will output a set $S$ such that $s \in S$.
\end{claim}
\begin{proof}
Since $s$ is an allowable split, both $V_a$ and $V_b$ are connected in $G$. Let $e^{\prime}$ be an edge in $G$ with one endpoint in $V_a$ and the other in $V_b$.  Such edge is guaranteed to exist by the connectivity of $G$. Let $G_a, G_b$ be the subgraphs of $G$ induced by $V_a, V_b$ (respectively) and let $T_a$, $T_b$ be spanning trees over $G_a, G_b$ (respectively).  Let $T^{\prime}$ be formed by adding $e^{\prime}$ to $T_a \cup T_b$.  It is straightforward to show that $T^{\prime}$ is a spanning tree over $G$.  Since $T^{\prime}$ is chosen w.p.p. by Wilson's algorithm, and removing $e^{\prime}$ results in the split $(V_a, V_b)$, w.p.p Algorithm~\ref{alg2} will return a set $S$ that contains $s$.
\end{proof}

\section{Experiments}
\label{experiments}
We perform experiments for two different prediction problems that contain categorical variables with non-ordinal structure.  The first problem involves predicting the probability of precipitation in different counties in the state of California.  The second involves predicting the probability of having an umbrella insurance policy using the US state of residence, in addition to other variables.

\subsection{Experimental Framework}

\begin{figure}[b]
\centering
\includegraphics[scale=0.3]{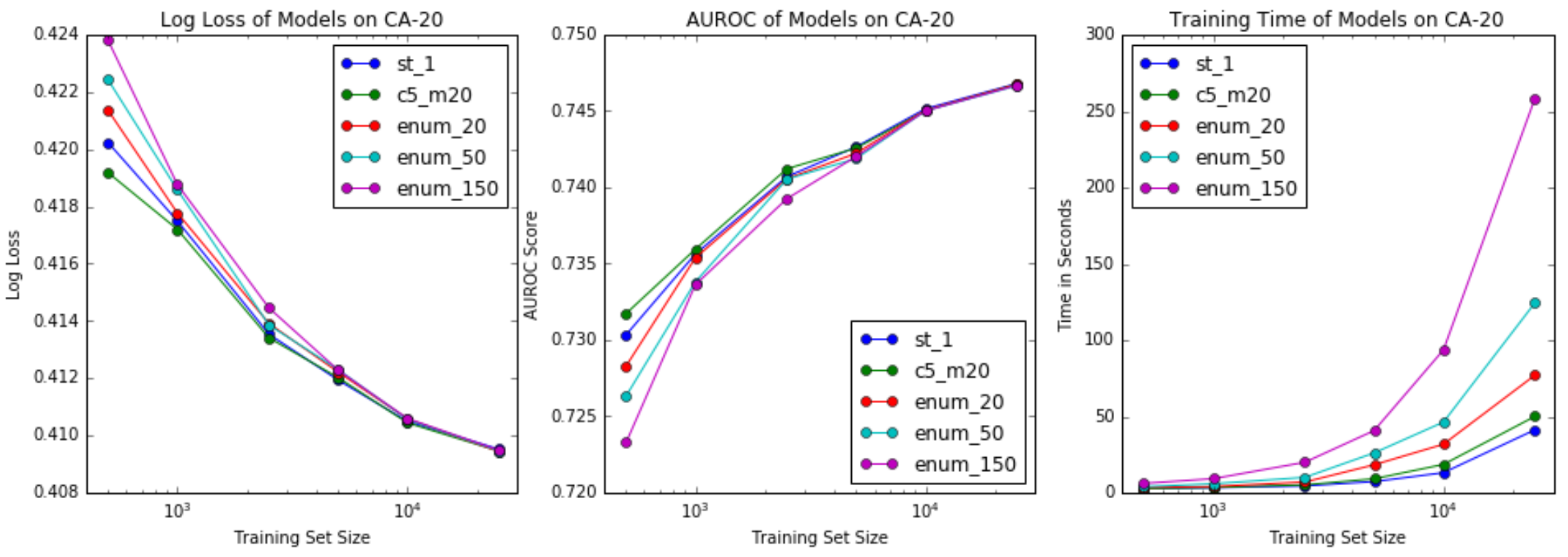}
\caption{Edge Contraction and Spanning Tree variants vs full enumeration on CA-20 dataset}
\label{ca-20-results}
\end{figure}

To fairly compare the models, in each of the experiments we take the following approach:

\begin{enumerate}
\item{Perform at least 5 trials of randomly splitting data into train, test, and validation components.}
\item{Choose training sets of different sizes from the train component.}
\item{Train all the contending models on training data with a low learning rate (.02), using the validation set to determine early stopping.}
\item{Try a wide range of {\em max\_depth} values (confirming by inspection that it was wide enough).}
\item{ For each particular training set size, choose the {\em max\_depth} value that yields the best performance averaged across all trials to represent the model.}
\item{Report the average training time for the relevant {\em max\_depth} value averaged across all trials.}
\item{Perform a paired t-test across the trials to assess statistical significance.}
\end{enumerate}

We use log-loss as our primary metric, but also report results on the area under the ROC curve (AUROC).  All experiments were performed on a 2019 MacBook Pro with a 2.4GHz Intel i9 processor.

\subsection{CA Weather Data}
The first set of experiments involve NOAA weather data set used in~\cite{categstruct}.  Each row corresponds to the weather on a particular day at a particular weather station, and we observe the month, county, and whether or not it rained on that day.   The goal is to accurately predict the probability of rain given the county and the month.  Since geographical proximity should imply similar precipitation patterns, there is good reason to expect that methods that can exploit this knowledge should outperform those that cannot, particularly in data-sparse scenarios.

\begin{figure}[t]
\centering
\includegraphics[scale=0.3]{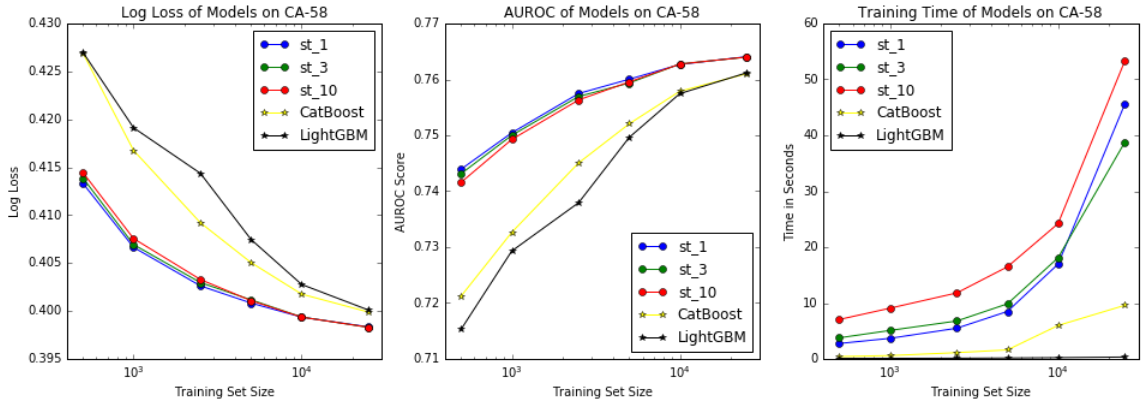}
\caption{Spanning Tree variants vs. CatBoost and LightGBM on CA-58}
\label{ca-58-results-st}
\end{figure}

\begin{figure}[b]
\centering
\includegraphics[scale=0.3]{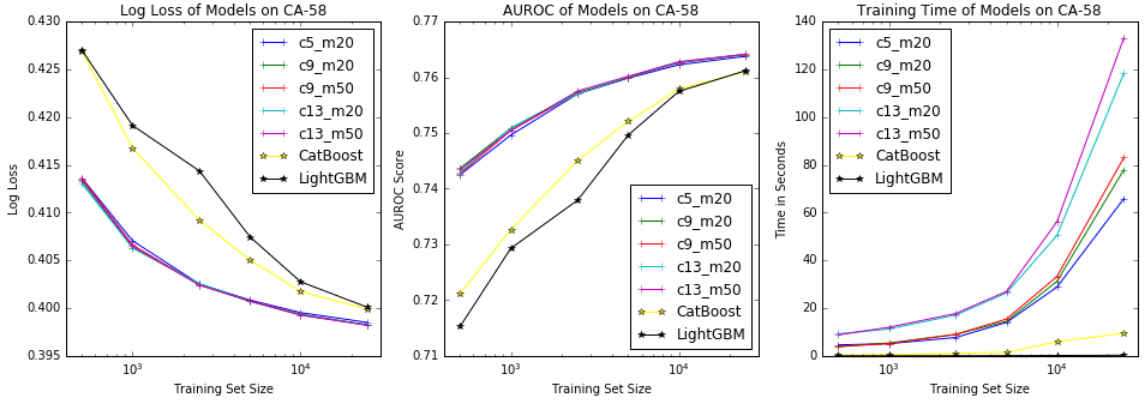}
\caption{Edge contraction variants vs. CatBoost and LightGBM on CA-58}
\label{ca-58-results-ec}
\end{figure}

\subsubsection{CA-20 Dataset}
 We begin by restricting ourselves to the 20-county subset explored in~\cite{categstruct}. Our goal is to compare the fully enumerated approaches from~\cite{categstruct} with the edge contraction and spanning tree approaches to split sampling described above.   We had a total of 140,135 data points, and so used validation and test sets of size 25,000 and training sets with sizes ranging from 500 to 25,000.   We compare the fully enumerated versions evaluating 20, 50, and 150 splits to two StructureBoost variants: one using the spanning tree approach with a single spanning tree (denoted ``st\_1") and the other using the edge contraction variant with contraction size $c$ equal to 5 and  the maximum number of splits to search $m$ equal to 20.  The results are in Figure~\ref{ca-20-results}.  As was shown in~\cite{categstruct}, we see that for the fully enumerated variants, evaluating just 20 splits outperforms evaluating 50 or 150 splits.  However, we also see that the two StructureBoost variants outperform the {\em enum\_20}.  This is quite interesting, as both of those variants evaluate about the same number of splits.  This suggests that choosing uniformly from the space of all splits may not, in fact, be the best approach.  Rather, the StructureBoost split sampling methods may serve to regularize further by altering the sampling distribution of allowable splits.

We also see in Figure~\ref{ca-20-results} that the StructureBoost variants are considerably quicker, which the spanning tree variant considerably faster than the edge contraction approach.  However, in this problem, the edge contraction variant performed slightly better, particularly for smaller training sets.

\subsubsection{CA-58 Dataset}
Next, we ran a similar experiment using the entire 58 counties of California.  Performing full enumeration on this graph is computationally prohibitive. Instead, we compare multiple variants of StructureBoost to CatBoost and LightGBM.  Figure~\ref{ca-58-results-st} compares the results of 3 spanning tree variants of StructureBoost to CatBoost and LightGBM while Figure~\ref{ca-58-results-ec} does the same for the edge-contraction variants of StructureBoost.  Notably, all variants of StructureBoost considerably outperform the unstructured alternatives, with relatively little performance variation amongst the StructureBoost variants.  The discrepancy is particularly large for smaller training set sizes -- a performance boost of more than .02 on AUC, merely by incorporating structure.  We ran paired t-tests between each StructureBoost variant and each of CatBoost and LightGBM to assess statistical significance at each training set size.  All comparisons were significant with a p-value of less than .01.

However, CatBoost and LightGBM are much faster than StructureBoost.  To some degree this reflect the necessary complexity required to deal with the categorical structure.  However, some of the discrepancy may also be due to better engineering and code optimization of the more mature boosting packages.  Amongst the StructureBoost variants, the spanning tree approach was considerably quicker than edge contraction.  On training sets of size 25000, none of the methods took more than a few minutes to train, and the spanning tree approaches generally took less than a minute.  Since the StructureBoost variant with a single spanning tree was fastest while still yielding excellent predictive power, we use that variant as the default StructureBoost for the experiments that follow.

\subsubsection{Predicting on Unseen Categorical Values}
A common problem in making predictions using categorical variables is how to make a prediction when a predictor has a value that is not present in the training set.  Unseen numerical values are common, and usually present no problem --  such interpolation (or extrapolation) is the essence of most supervised learning techniques.  But without a representation of the structure of a categorical variable, there can be no notion of interpolation on categorical values.  However, categorical structure should enable methods like StructureBoost to effectively do interpolation on categorical variables.

To explore this notion further, we performed the following experiment, again using the CA-58 dataset.  We omitted 3 counties (Riverside, Mariposa, and Trinity) from the training data and chose training sets from the remaining counties.  Validation sets of size 25,000 were also chosen from the remaining counties, but test sets of size 5,000 were chosen only from omitted counties.  The goal was to see how well the various models could do when predicting on counties for which they had no training data.  The results are in Figure~\ref{fig:ca-58-results-rem}.  We see that StructureBoost markedly outperforms both CatBoost and LightGBM. The paired t-test gave p-values of less than .01 for all training set sizes for StructureBoost against the alternatives.

\begin{figure}[t]
\begin{subfigure}{.33\textwidth}
  \centering
  \includegraphics[width=.8\linewidth]{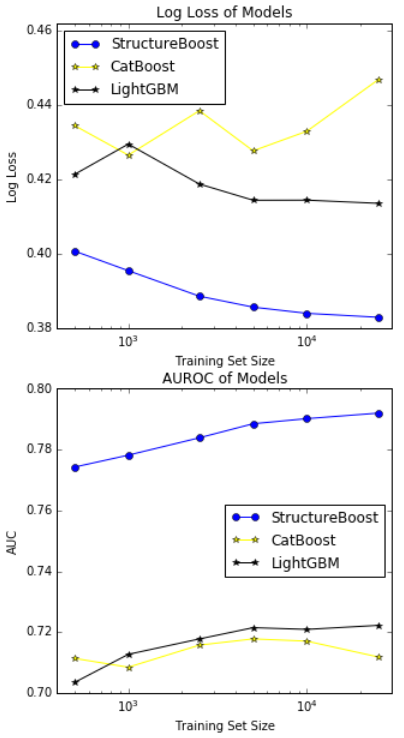}  
  \caption{Performance of models }
  \label{fig:ca-58-results-rem}
\end{subfigure}
\begin{subfigure}{.6\textwidth}
  \centering
  \includegraphics[width=.8\linewidth]{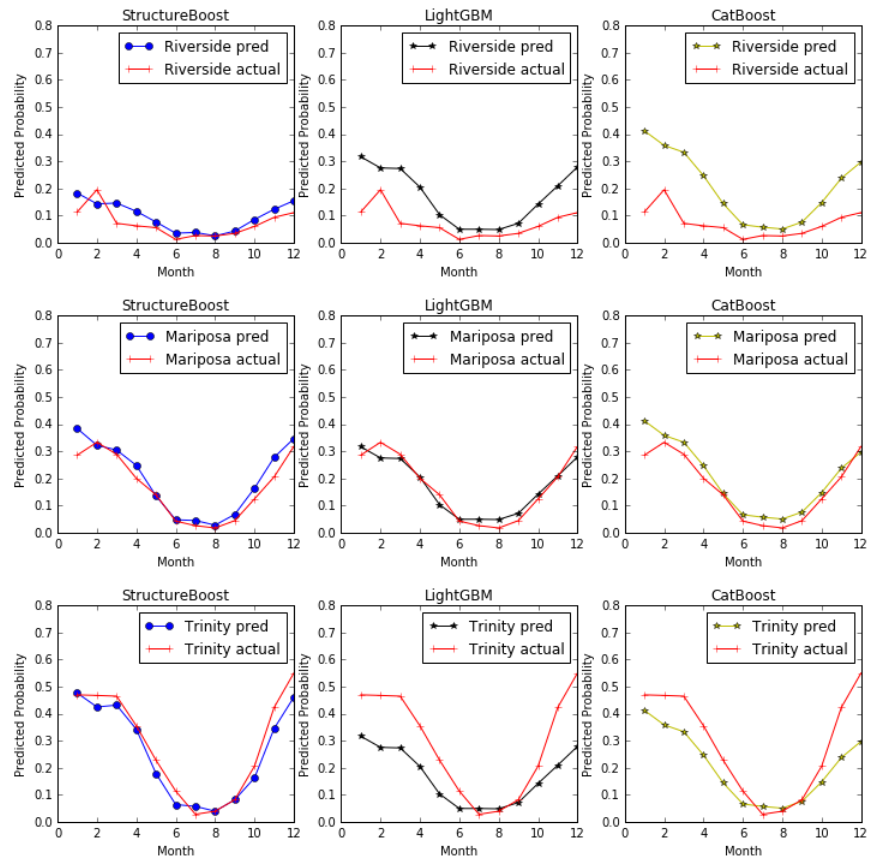}  
  \caption{Predicted vs Actual probabilities}
  \label{fig:prob-compare}
\end{subfigure}
\caption{Results when trained on 55 counties and tested on other 3}
\label{fig:fig}
\end{figure}

It is instructive to compare the predicted probabilities of the 3 models for the 3 omitted counties and compare them to the ``actuals" (computed using all the data).  This is shown in Figure~\ref{fig:prob-compare}, for a training set of size 2,500.  In the leftmost column, we see that the monthly probabilities predicted by StructureBoost (blue curves) closely approximate the actual probabilities (red curves) in each of the three omitted counties.  However, in the following two columns, both LightGBM (black curves) and CatBoost (yellow curves) make the same predictions for all three counties.  Consequently, they over-predict the probability of rain in Riverside county and under-predict in Trinity county.

\subsection{Umbrella Insurance Prediction}
Umbrella insurance is a type of additional liability insurance that some people carry in addition to standard auto and home insurance policies.  It is typically carried by older and wealthier people who have substantial assets to protect.  Insurance in the U.S. is regulated at the state level, therefore costs and coverages of policies can vary substantially from state to state.  Unlike the weather example, however, it is not obvious that adjacency of states would imply similarity in tendency to buy umbrella insurance.  However, the state of residence may be a proxy for demographics, attitudes toward risk, and aspects of regulation that may be more similar in neighboring states. 

\begin{figure}[t]
\centering
\includegraphics[scale=0.28]{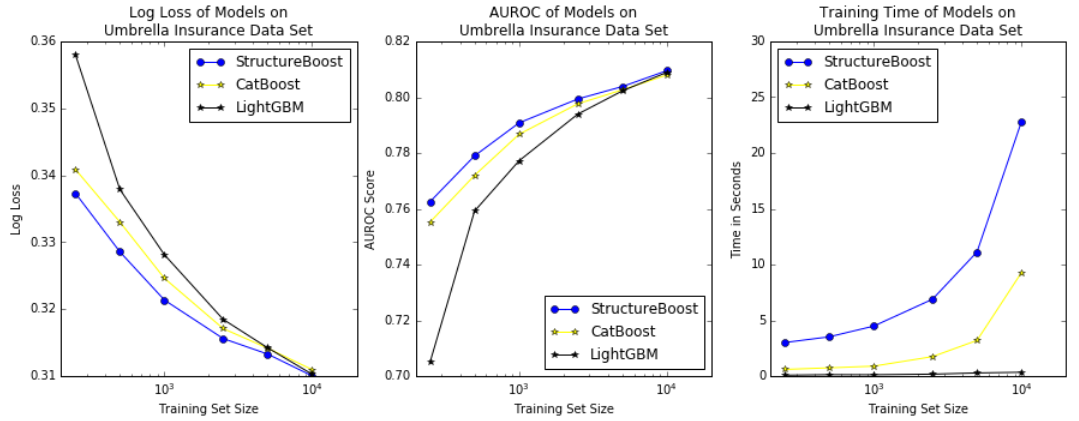}
\caption{Comparison of different models on Umbrella prediction.}
\label{umbrella-results}
\end{figure}

We obtained a dataset of 24,742 different insurance policyholders that had both home and auto insurance in the lower 48 US States.  We obtained their age, two different home insurance coverage limits (one for the dwelling, one for personal property), their state of residence, and whether or not they have an umbrella policy.  We compared StructureBoost (with a single spanning tree) to CatBoost and LightGBM to see how well the different models could predict who would have an umbrella policy from these variables.  We ran 25 trials, with a test set of size 10,000, a validation set of size 4,742, and training set sizes ranging from 250 to 10,000. The results are in Figure~\ref{umbrella-results}.  Again, we see that StructureBoost outperforms the unstructured alternatives.  While the discrepancies were not as large as in the previous examples - this is to be expected.  As noted earlier, the value of structure in this problem is less apparent than in the previous ones.  Nevertheless, paired t-tests indicate significance with p-values less than .01 for all comparisons, except for the comparison with LightGBM at size 10,000 which had a p-value of .115.

\section{Summary and Conclusions}
We demonstrate that Gradient Boosting can be done efficiently and effectively with complex structured categorical variables, and that exploiting this structure yields superior predictive performance.  We also shed some light on {\em why} structure is important --  since it enables accurate prediction on unseen categorical values, it certainly must aid performance in realms of data sparsity. 

While the structured categorical variables used in this paper referred to spatial and geographical considerations, the ability to encode graphical structure in a predictor is not limited to these situations.  For example, edit distance can provide a graph structure on strings as suggested in~\cite{bilisoly2014generalizing}.  This approach may be useful for text analysis and bioinformatics problems.   The StructureBoost package in Python is open-source and freely available as is the specific code and data used for the experiments.

There are several directions of future research that stem from the notion on categorical structure and StructureBoost specifically.  The first involves improving predictive performance in a wide variety of domains where categorical variables are being used as predictors.  A second direction involves a ``data-driven" approach to edge contraction -- can we choose the edges in a way to consolidate ``data-poor" regions together to strengthen the signal?  Next, can we apply the notion of categorical structure to the target variable to improve performance in multi-class classification problems?  For example, predicting that an image is a chimpanzee when it is actually a monkey is less ``wrong" than predicting either of them is a car.  However, to make that kind of judgement, performance metrics must take into account the structure of the target variable.  Finally, we believe there are interesting questions around the mathematical foundations of categorical variables that can be explored.

\section*{Broader Impact}

The contributions in this paper, and the StructureBoost package in general, provide the capability to better utilize categorical variables to make predictions.  Like any statistical technique, when used properly, it could enable discoveries, tools, or products that help people.  Conversely, if used improperly, it could have a negative impact.  One particular avenue is worth mentioning.  Race, ethnicity, and other demographic variables are categorical, and we could imagine this method being used on such variables.  If used properly, this could be one way to reduce disparities in predictive performance across classes that have occurred in various scenarios.  However, one could also imagine situations where techniques like this are used to justify unwarranted conclusions, or otherwise penalize members of disadvantaged groups.  There are a priori assumptions that one is making when one attributes a graphical structure to the values of a variable.  These could easily be mis-used or mis-construed.

\bibliography{stb_neurips_2020}

\begin{thebibliography}{15}
\providecommand{\natexlab}[1]{#1}
\providecommand{\url}[1]{\texttt{#1}}
\expandafter\ifx\csname urlstyle\endcsname\relax
  \providecommand{\doi}[1]{doi: #1}\else
  \providecommand{\doi}{doi: \begingroup \urlstyle{rm}\Url}\fi

\bibitem[Bilisoly(2014)]{bilisoly2014generalizing}
R.~Bilisoly.
\newblock Generalizing the mean and variance to categorical data using metrics.
\newblock \emph{arXiv preprint arXiv:1410.1106}, 2014.

\bibitem[Chen and Guestrin(2016)]{chen2016xgboost}
T.~Chen and C.~Guestrin.
\newblock Xgboost: A scalable tree boosting system.
\newblock In \emph{Proceedings of the 22nd acm sigkdd international conference
  on knowledge discovery and data mining}, pages 785--794. ACM, 2016.

\bibitem[Dorogush et~al.(2018)Dorogush, Ershov, and
  Gulin]{dorogush2018catboost}
A.~V. Dorogush, V.~Ershov, and A.~Gulin.
\newblock Catboost: gradient boosting with categorical features support.
\newblock \emph{arXiv preprint arXiv:1810.11363}, 2018.

\bibitem[Elbassioni(2015)]{elbassioni2015polynomial}
K.~M. Elbassioni.
\newblock A polynomial delay algorithm for generating connected induced
  subgraphs of a given cardinality.
\newblock \emph{J. Graph Algorithms Appl.}, 19\penalty0 (1):\penalty0 273--280,
  2015.

\bibitem[Friedman(2000)]{Friedman00greedyfunction}
J.~H. Friedman.
\newblock Greedy function approximation: A gradient boosting machine.
\newblock \emph{Annals of Statistics}, 29:\penalty0 1189--1232, 2000.

\bibitem[Friedman(2002)]{Friedman:2002:SGB:635939.635941}
J.~H. Friedman.
\newblock Stochastic gradient boosting.
\newblock \emph{Comput. Stat. Data Anal.}, 38\penalty0 (4):\penalty0 367--378,
  Feb. 2002.
\newblock ISSN 0167-9473.
\newblock \doi{10.1016/S0167-9473(01)00065-2}.
\newblock URL \url{http://dx.doi.org/10.1016/S0167-9473(01)00065-2}.

\bibitem[Ke et~al.(2017)Ke, Meng, Finley, Wang, Chen, Ma, Ye, and
  Liu]{ke2017lightgbm}
G.~Ke, Q.~Meng, T.~Finley, T.~Wang, W.~Chen, W.~Ma, Q.~Ye, and T.-Y. Liu.
\newblock Lightgbm: A highly efficient gradient boosting decision tree.
\newblock In \emph{Advances in Neural Information Processing Systems}, pages
  3146--3154, 2017.

\bibitem[Komusiewicz and Sommer(2019)]{komusiewicz2019enumerating}
C.~Komusiewicz and F.~Sommer.
\newblock Enumerating connected induced subgraphs: Improved delay and
  experimental comparison.
\newblock In \emph{International Conference on Current Trends in Theory and
  Practice of Informatics}, pages 272--284. Springer, 2019.

\bibitem[Lucena(2020)]{categstruct}
B.~Lucena.
\newblock Exploiting categorical structure using tree-based methods.
\newblock In \emph{Proceedings of the 23rd International Conference on
  Artificial Intelligence and Statistics, AISTATS}, volume 108 of
  \emph{Proceedings of Machine Learning Research}, pages 2949--2958, 2020.

\bibitem[Prokhorenkova et~al.(2018)Prokhorenkova, Gusev, Vorobev, Dorogush, and
  Gulin]{NIPS2018_7898}
L.~Prokhorenkova, G.~Gusev, A.~Vorobev, A.~V. Dorogush, and A.~Gulin.
\newblock Catboost: unbiased boosting with categorical features.
\newblock In S.~Bengio, H.~Wallach, H.~Larochelle, K.~Grauman, N.~Cesa-Bianchi,
  and R.~Garnett, editors, \emph{Advances in Neural Information Processing
  Systems 31}, pages 6638--6648. Curran Associates, Inc., 2018.
\newblock URL
  \url{http://papers.nips.cc/paper/7898-catboost-unbiased-boosting-with-categorical-features.pdf}.

\bibitem[Robertson and Seymour(1986)]{robertson1986graph}
N.~Robertson and P.~D. Seymour.
\newblock Graph minors. v. excluding a planar graph.
\newblock \emph{Journal of Combinatorial Theory, Series B}, 41\penalty0
  (1):\penalty0 92--114, 1986.

\bibitem[Robertson and Seymour(1990)]{robertson1990graph}
N.~Robertson and P.~D. Seymour.
\newblock Graph minors. iv. tree-width and well-quasi-ordering.
\newblock \emph{Journal of Combinatorial Theory, Series B}, 48\penalty0
  (2):\penalty0 227--254, 1990.

\bibitem[Schild(2017)]{DBLP:journals/corr/abs-1711-06455}
A.~Schild.
\newblock An almost-linear time algorithm for uniform random spanning tree
  generation.
\newblock \emph{CoRR}, abs/1711.06455, 2017.
\newblock URL \url{http://arxiv.org/abs/1711.06455}.

\bibitem[Weisstein(2019)]{USgraph}
E.~W. Weisstein.
\newblock Contiguous usa graph. from mathworld--a wolfram web resource., 2019.
\newblock URL \url{http://mathworld.wolfram.com/ContiguousUSAGraph.html}.

\bibitem[Wilson(1996)]{wilson1996generating}
D.~B. Wilson.
\newblock Generating random spanning trees more quickly than the cover time.
\newblock In \emph{Proceedings of the twenty-eighth annual ACM symposium on
  Theory of computing}, pages 296--303, 1996.

\end{thebibliography}

\end{document}